\documentclass[english,letter]{article}
\usepackage{times}
\usepackage{graphicx} 
\usepackage{subcaption} 
\usepackage{natbib}
\usepackage{algorithm}
\usepackage{algorithmic}
\usepackage{hyperref}
\usepackage{tikz}
\usepackage{amsmath}
\usepackage{amssymb}
\usepackage{amsthm}
\usepackage{multicol}
\usepackage{thmtools,thm-restate}
\usepackage{nameref}
\usepackage{array}
\usepackage{paralist}
\usepackage{placeins}
\usetikzlibrary{shapes,snakes}

\newcolumntype{C}[1]{>{\centering\arraybackslash}m{#1}}

\usepackage[accepted]{icml2017}

\global\long\def\E{\mathrm{E}}
\global\long\def\Var{\mathrm{Var}}
\global\long\def\KL{\mathrm{KL}}

\newtheorem{proposition}{Proposition}

\newtheorem{corollary}[proposition]{Corollary}

\numberwithin{equation}{section}

\icmltitlerunning{Adversarial Variational Bayes}

\begin{document}
\twocolumn[
\icmltitle{Adversarial Variational Bayes:\\Unifying Variational Autoencoders and Generative Adversarial Networks}
\begin{icmlauthorlist}
\icmlauthor{Lars Mescheder}{avg}\qquad
\icmlauthor{Sebastian Nowozin}{msr}\qquad
\icmlauthor{Andreas Geiger}{avg,cvg}
\end{icmlauthorlist}
\icmlaffiliation{avg}{Autonomous Vision Group, MPI T\"ubingen}
\icmlaffiliation{msr}{Microsoft Research Cambridge}
\icmlaffiliation{cvg}{Computer Vision and Geometry Group, ETH Z\"urich}
\icmlcorrespondingauthor{Lars Mescheder}{lars.mescheder@tuebingen.mpg.de}

\icmlkeywords{machine learning, GAN, VAE, generative models}
\vskip 0.3in
]

\printAffiliationsAndNotice{}

\begin{abstract}
Variational Autoencoders (VAEs) are expressive latent variable models that can be used to learn complex
probability distributions from training data. However, the quality of the resulting model
crucially relies on the expressiveness of the inference model.
We introduce Adversarial Variational Bayes
(AVB), a technique for training Variational Autoencoders 
with arbitrarily expressive inference models.
We achieve this by introducing an auxiliary discriminative network that
allows to rephrase the maximum-likelihood-problem as a two-player game, hence establishing 
a principled connection between VAEs and Generative Adversarial Networks (GANs). 
We show that in the nonparametric limit our method yields an
\emph{exact} maximum-likelihood assignment for the parameters of the
generative model, as well as the \emph{exact} posterior distribution
over the latent variables given an observation. Contrary to competing
approaches which combine VAEs with GANs, our approach has a clear theoretical justification, retains
most advantages of standard Variational Autoencoders and is easy to
implement.
\end{abstract}
\section{Introduction}
Generative models in machine learning are models that can be trained
on an unlabeled dataset and are capable of generating new data points
after training is completed. As generating new content requires a
good understanding of the training data at hand, such models are often regarded
as a key ingredient to unsupervised learning.
\begin{figure}[t!]

\input{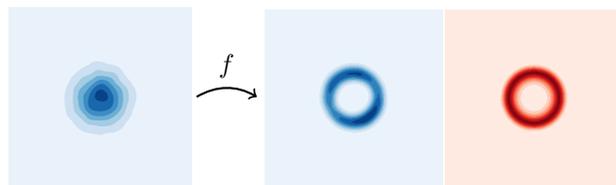}

\vspace{-0.8cm}
\caption{\label{fig:donut}
We propose a method which enables neural samplers with intractable density for Variational Bayes and as inference models for learning latent variable models. This toy example demonstrates our method's ability to accurately approximate complex posterior distributions like the one shown on the right.
}
\vspace{-0.5cm}
\end{figure}

In recent years, generative models have become more and more powerful. While many model classes such as PixelRNNs \cite{oord2016pixel},
PixelCNNs \cite{van2016conditional}, real NVP \cite{dinh2016density}
and Plug \& Play generative networks \cite{nguyen2016plug} have been
introduced and studied, the two most prominent ones
are Variational Autoencoders (VAEs) \cite{kingma2013auto,rezende2014stochastic}
and Generative Adversarial Networks (GANs) \cite{goodfellow2014generative}.

Both VAEs and GANs come with their own advantages and disadvantages: while GANs generally
yield visually sharper results when applied to learning a representation of natural images, VAEs are attractive because they naturally yield both a generative model and an inference model.
Moreover, it was reported, that VAEs often lead to better log-likelihoods \cite{wu2016quantitative}. The recently introduced BiGANs \cite{donahue2016adversarial,dumoulin2016adversarially} add an inference model to GANs. However, it was observed that the reconstruction results often only vaguely resemble the input
and often do so only semantically and not in terms of pixel values.

The failure of VAEs to generate sharp images is often attributed to the fact that the inference models used during training are usually not expressive enough to
capture the true posterior distribution.
Indeed, recent work shows that using more expressive model classes can lead to substantially better results \cite{kingma2016improving}, both visually
and in terms of log-likelihood bounds. Recent work \cite{chen2016variational} also suggests that highly expressive inference models are
essential in presence of a strong decoder to allow the model to make use of  the latent space at all. 

In this paper, we present Adversarial Variational Bayes (AVB)
\footnote{
Concurrently to our work, several researchers have described similar ideas.
Some ideas of this paper were described independently by Husz\'ar in a blog post on \url{http://www.inference.vc} and in \citet{huszar2017variational}. The idea to use adversarial training to improve the encoder network was also suggested by Goodfellow in an exploratory talk he gave at NIPS 2016 and by \citet{li2016wild}.
A similar idea was also mentioned by \citet{karaletsos2016adversarial} in the context of message passing in graphical models.
}
, a technique for training
Variational Autoencoders with arbitrarily flexible inference models parameterized by neural networks.
We can show that in the nonparametric limit we obtain a maximum-likelihood assignment for the generative model together
with the correct posterior distribution. 

While there were some attempts at combining VAEs and GANs \cite{makhzani2015adversarial,larsen2015autoencoding},
most of these attempts are not motivated from a maximum-likelihood point
of view and therefore usually do not lead to maximum-likelihood
assignments. For example, in Adversarial Autoencoders (AAEs) \cite{makhzani2015adversarial} the Kullback-Leibler regularization term that appears in the training 
objective for VAEs is replaced
with an adversarial loss that encourages the aggregated posterior to be close to the prior over the latent variables. Even though AAEs 
do \emph{not} maximize a lower bound to the maximum-likelihood 
objective, we show in Section \ref{sec:related-aae} that
AAEs can be interpreted as an approximation to our approach, thereby establishing a connection of AAEs to maximum-likelihood learning.

Outside the context of generative models, AVB yields a new method for performing Variational Bayes (VB) with neural samplers. 
This is illustrated in Figure \ref{fig:donut}, where we used AVB to train a neural network  to sample from a non-trival unnormalized probability density.
This allows to accurately approximate the posterior distribution of a probabilistic model, e.g. for Bayesian parameter estimation.
The only other variational methods we are aware of that can deal with such expressive inference models are based on Stein Discrepancy \cite{ranganath2016operator, liu2016two}. However, those methods usually do not directly target the reverse Kullback-Leibler-Divergence and can therefore
not be used to approximate the variational lower bound for learning a latent variable model.

Our contributions are as follows:
\begin{compactitem}
 \item We enable the usage of arbitrarily complex inference models for Variational Autoencoders using adversarial training.
 \item We give theoretical insights into our method, showing that in the nonparametric limit our method recovers the true posterior distribution as well as a true maximum-likelihood assignment for the parameters of the 
       generative model.
 \item We empirically demonstrate that our model is able to learn rich posterior distributions and show that the model is able to generate compelling samples for complex data sets.
\end{compactitem}

\section{Background}

\begin{figure}
\centering
\begin{subfigure}[b]{0.45\linewidth}
\centering
\resizebox{\hsize}{!}{
 \begin{tikzpicture}[scale=0.75]

\tikzstyle{edge}=[ultra thick]
\tikzstyle{vertex}=[circle,fill=white,draw, minimum width=1.1cm]
\tikzstyle{function}=[draw,minimum width=1.1cm]

\draw (-4, -4) edge[dashed] (4, -4) node[above right]{Encoder};
\draw (-4, -10) edge[dashed] (4, -10) node[above right]{Decoder};

 \node[vertex] (x)  at (0, 1) {$x$} ;
 \node[function] (f)  at (0, -1) {$f$} ;
 \node[vertex] (eps)  at (2,  -1) {$\epsilon_1$} ;
 \node[function] (plus1)  at (0, -2.5) {$+/*$} ;
 \node[vertex] (z)  at (0, -4) {$z$} ;

 \node[function] (g)  at (0, -6) {$g$} ;
 \node[vertex] (eps2)  at (2, -6) {$\epsilon_2$} ;
 \node[function] (plus2)  at (0, -7.5) {$+/*$} ;
 \node[vertex] (x2)  at (0, -9) {$x$} ;

\draw (x) edge [edge, ->] (f);
\draw (eps) edge [edge, ->] (plus1);
\draw (f) edge [edge, ->] (plus1);
\draw (plus1) edge [edge, ->] (z);

\draw (z) edge[edge, ->] (g);
\draw (g) edge[edge, ->] (plus2);
\draw (eps2) edge[edge, ->] (plus2);
\draw (plus2) edge[edge, ->] (x2);

\end{tikzpicture}
}
\caption{Standard VAE}
\label{fig:inference-comparison-standard}
\end{subfigure}
\begin{subfigure}[b]{0.45\linewidth}
\centering
\resizebox{\hsize}{!}{
 \begin{tikzpicture}[scale=0.75]

\tikzstyle{edge}=[ultra thick]
\tikzstyle{vertex}=[circle,fill=white,draw, minimum width=1.1cm]
\tikzstyle{function}=[draw,minimum width=1.1cm]

\draw (-4, -4) edge[dashed] (4, -4) node[above right]{Encoder};
\draw (-4, -10) edge[dashed] (4, -10) node[above right]{Decoder};

 \node[vertex] (x)  at (-1, 0) {$x$} ;
 \node[vertex] (eps)  at (1, 0) {$\epsilon_1$} ;
 \node[function] (f)  at (0, -2) {$f$} ;
 \node[vertex] (z)  at (0, -4) {$z$} ;
 \node[function] (g)  at (0, -6) {$g$} ;
 \node[vertex] (eps2)  at (2, -6) {$\epsilon_2$} ;
 \node[function] (plus)  at (0, -7.5) {$+/*$} ;
 \node[vertex] (x2)  at (0, -9) {$x$} ;

\draw (eps) edge [edge, ->] (f);
\draw (x) edge [edge, ->] (f);
\draw (f) edge [edge, ->] (z);
\draw (z) edge[edge, ->] (g);
\draw (z) edge[edge, ->] (g);
\draw (g) edge[edge, ->] (plus);
\draw (eps2) edge[edge, ->] (plus);
\draw (plus) edge[edge, ->] (x2);

\end{tikzpicture}
}
\caption{Our model}
\label{fig:inference-comparison-blackbox}
\end{subfigure}
\caption{Schematic comparison of a standard VAE and a VAE with black-box inference
model, where $\epsilon_1$ and $\epsilon_2$ denote samples from some noise distribution. While more complicated inference models for Variational Autoencoders
are possible, they are usually not as flexible as our black-box approach.}
\label{fig:inference-comparison}
\vspace{-0.4cm}
\end{figure}
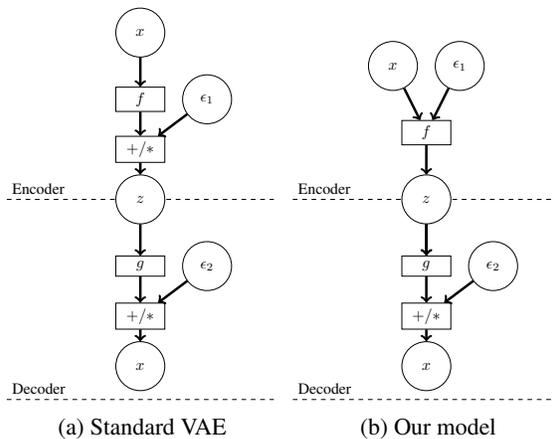

As our model is an extension of Variational Autoencoders (VAEs) \cite{kingma2013auto,rezende2014stochastic}, we start with a brief review of VAEs.

VAEs
are specified by a parametric generative model $p_{\theta}(x\mid z)$
of the visible variables given the latent variables, a prior $p(z)$
over the latent variables and an approximate inference model $q_{\phi}(z\mid x)$
over the latent variables given the visible variables. It can be shown
that
\begin{multline}\label{eq:variational-equation}
\log p_\theta(x)\geq-\KL(q_{\phi}(z\mid x),p(z))\\
+\E_{q_{\phi}(z\mid x)}\log p_{\theta}(x\mid z).
\end{multline}
The right hand side of \eqref{eq:variational-equation} is called the variational lower bound or 
evidence lower bound (ELBO).  
If there is $\phi$ such that $q_{\phi}(z\mid x)=p_{\theta}(z\mid x)$,
we would have
\begin{multline}\label{eq:variational-equation-max-eq}
\log p_\theta(x)=\max_{\phi}-\KL(q_{\phi}(z\mid x),p(z))\\
+\E_{q_{\phi}(z\mid x)}\log p_{\theta}(x\mid z).
\end{multline}
However, in general this is not true, so that we only have an inequality
in \eqref{eq:variational-equation-max-eq}.

When performing maximum-likelihood training, our goal is to optimize
the marginal log-likelihood
\begin{equation}\label{eq:MLE-problem}
 \E_{p_\mathcal D(x)} \log p_\theta (x), 
\end{equation}
where $p_{\mathcal D}$ is the data distribution.
Unfortunately, computing  $\log p_\theta (x)$ requires marginalizing out $z$ in $p_\theta(x,z)$
which is usually intractable. Variational Bayes uses inequality \eqref{eq:variational-equation} to 
rephrase the intractable problem of optimizing \eqref{eq:MLE-problem} into
\begin{multline}\label{eq:VAE-objective}
 \max_\theta \max_\phi \E_{p_{\mathcal D}(x)}\Bigl[
 -\KL(q_{\phi}(z\mid x),p(z)) \\
+ \E_{q_{\phi}(z\mid x)}\log p_{\theta}(x\mid z)
\Bigr].
\end{multline}
Due to inequality \eqref{eq:variational-equation}, we still optimize a lower bound to the true maximum-likelihood objective
\eqref{eq:MLE-problem}.

Naturally, the quality of this lower bound depends on the expressiveness of the inference model $q_\phi(z\mid x)$. Usually, 
$q_\phi(z\mid x)$ is taken to be a Gaussian distribution with diagonal covariance matrix whose mean and variance vectors are parameterized 
by neural networks with $x$ as input \cite{kingma2013auto, rezende2014stochastic}. While this model is very flexible in its dependence on $x$, its dependence on $z$ is very restrictive, potentially
limiting the quality of the resulting generative model. Indeed, it was observed that applying standard Variational Autoencoders to natural images
often results in blurry images \cite{larsen2015autoencoding}.

\section{Method}\label{sec:avb}
In this work we show how we can instead use a black-box inference model
$q_{\phi}(z\mid x)$ and use adversarial training to obtain an approximate maximum
likelihood assignment $\theta^{*}$ to $\theta$ and a close approximation $q_{\phi^*}(z\mid x)$
to the true posterior $p_{\theta^{*}}(z\mid x)$. This is visualized in Figure \ref{fig:inference-comparison}: on the left hand side the
structure of a typical VAE is shown. The right hand side shows our flexible black-box inference model. In contrast to a VAE with Gaussian inference model, we include the noise $\epsilon_1$ as additional input to the inference model instead of adding it at the very end, thereby allowing the inference network to learn complex probability distributions.

\subsection{Derivation}\label{sec:avb-derivation}
To derive our method, we rewrite the optimization problem in \eqref{eq:VAE-objective}
as
\begin{multline}\label{eq:variational-equation-reformulation}
 \max_\theta \max_\phi \E_{p_\mathcal D(x)} \E_{q_{\phi}(z\mid x)}\big(\log p(z) \\
-\log q_{\phi}(z\mid x)+\log p_{\theta}(x\mid z)\big).
\end{multline}
When we have an explicit representation of $q_\phi(z \mid x)$ such as a Gaussian parameterized by a neural network, \eqref{eq:variational-equation-reformulation} 
can be optimized using the reparameterization trick \cite{kingma2013auto,rezende2015variational} and stochastic gradient descent.
Unfortunately, this is not the case when we define $q_\phi(z \mid x)$ by a black-box procedure as illustrated in
Figure \ref{fig:inference-comparison-blackbox}.

The idea of our approach is to circumvent this problem by implicitly representing the term
\begin{equation}\label{eq:variational-equation-intractable-part}
  \log p(z) - \log q_\phi(z \mid x)
\end{equation}
as the optimal value of an additional real-valued discriminative network $T(x, z)$ that we introduce to the problem.

More specifically, consider the following objective for the discriminator $T(x, z)$ for a given $q_\phi(x\mid z)$: 
\begin{multline}\label{eq:discriminator-obj}
\max_{T}
\E_{p_{\mathcal{D}}(x)}\E_{q_{\phi}(z\mid x)} \log\sigma(T(x,z)) \\
+ \E_{p_{\mathcal{D}}(x)}\E_{p(z)} \log\left(1-\sigma(T(x,z))\right).
\end{multline}
Here, $\sigma(t) := (1 + \mathrm e^{-t})^{-1}$ denotes the sigmoid-function.
Intuitively, $T(x, z)$ tries to distinguish pairs $(x,z)$ that were sampled independently 
using the distribution $p_{\mathcal{D}}(x)p(z)$ from those that were
sampled using the current inference model, i.e., using $p_{\mathcal{D}}(x)q_{\phi}(z\mid x)$.

To simplify the theoretical analysis, we assume that the model $T(x,z)$ is flexible enough to
represent any function of the two variables $x$ and $z$. This assumption is often referred to as the nonparametric limit \cite{goodfellow2014generative}
and is justified by the fact
that deep neural networks are universal function approximators \cite{hornik1989multilayer}.

As it turns out, the optimal discriminator $T^*(x,z)$ according to the objective  in \eqref{eq:discriminator-obj} is given by the negative of
\eqref{eq:variational-equation-intractable-part}.
\begin{restatable}{proposition}{propoptimaldiscriminator}
\label{prop:optimal-discriminator}
 For $p_\theta(x \mid z)$ and $q_\phi(z \mid x)$ fixed, the optimal discriminator $T^*$ according to the objective
 in \eqref{eq:discriminator-obj} is given by
 \begin{equation}\label{eq:optimal-discriminator}
  T^*(x, z) = \log q_\phi(z \mid x) - \log p(z).
 \end{equation}
\end{restatable}

\begin{proof}
The proof is analogous to the proof of Proposition 1 in \citet{goodfellow2014generative}. See the Supplementary Material 

for details.
\end{proof}
\vspace{-0.4cm}
Together with \eqref{eq:variational-equation-reformulation}, Proposition \ref{prop:optimal-discriminator} allows us to write the optimization objective in
\eqref{eq:VAE-objective} as
\begin{equation}\label{eq:avb-objective}
\max_{\theta, \phi}\E_{p_{\mathcal{D}}(x)}\E_{q_{\phi}(z\mid x)}\big(-T^*(x, z)
 + \log p_{\theta}(x\mid z)\big),
\end{equation}
where $T^*(x, z)$ is defined as the function that maximizes \eqref{eq:discriminator-obj}.

To optimize \eqref{eq:avb-objective}, we need to calculate the gradients of \eqref{eq:avb-objective} with respect to $\theta$ and $\phi$. 
While taking the gradient with respect to $\theta$ is straightforward, taking the gradient with respect to $\phi$ is complicated by the
fact that we have defined $T^*(x, z)$ indirectly as the solution of an auxiliary optimization problem which itself depends on $\phi$. However, the following Proposition shows that taking
the gradient with respect to the explicit occurrence of $\phi$ in $T^*(x, z)$ is not necessary:
\begin{restatable}{proposition}{propgradientdiscriminatorzero}\label{prop:gradient-discriminator-zero}
We have
\begin{equation}\label{eq:gradient-discriminator-zero}
 \E_{q_{\phi}(z\mid x)}\left(\nabla_\phi T^*(x, z)\right) = 0 .
\end{equation}
\end{restatable}
\begin{proof}
 The proof can be found in the Supplementary Material.
\end{proof}

\vspace{-0.4cm}
Using the reparameterization trick \cite{kingma2013auto, rezende2014stochastic}, \eqref{eq:avb-objective}
can be rewritten in the form
\begin{multline}\label{eq:avb-objective-reparameterization}
\max_{\theta, \phi}\E_{p_{\mathcal{D}}(x)} \E_{\epsilon}\big(-T^*(x, z_{\phi}(x,\epsilon)) \\
 + \log p_{\theta}(x\mid z_{\phi}(x,\epsilon))\big)
\end{multline}
for a suitable function $z_{\phi}(x,\epsilon)$. Together with Proposition \ref{prop:optimal-discriminator}, \eqref{eq:avb-objective-reparameterization}
allows us to take unbiased estimates of the gradients of \eqref{eq:avb-objective} with respect to $\phi$ and $\theta$.

\subsection{Algorithm}\label{sec:avb-algorithm}
\begin{algorithm}[t!]
\caption{Adversarial Variational Bayes (AVB)}
\label{alg:avb}
\begin{algorithmic}[1]
    \STATE $i \gets 0$
   \WHILE{not converged}
     \STATE Sample $\{x^{(1)}, \dots, x^{(m)}\}$ from data distrib. $p_\mathcal D(x)$
     \STATE Sample $\{z^{(1)}, \dots, z^{(m)}\}$ from prior $p(z)$
     \STATE Sample $\{\epsilon^{(1)}, \dots, \epsilon^{(m)}\}$ from $\mathcal N(0, 1)$
     
     \STATE Compute $\theta$-gradient (eq. \ref{eq:avb-objective-reparameterization}):\\
     \vspace{1ex}$
     g_\theta \gets  \frac{1}{m}  \sum_{k=1}^m \nabla_\theta
     \log p_\theta \left(x^{(k)} \mid z_\phi\left(x^{(k)}, \epsilon^{(k)}\right) \right)
     $\vspace{1ex}
     
     \STATE Compute $\phi$-gradient (eq. \ref{eq:avb-objective-reparameterization}):\\
     \vspace{1ex}$
     g_\phi \gets  \frac{1}{m} \sum_{k=1}^m 
     \nabla_\phi \bigl[
     -T_\psi \left(x^{(k)}, z_\phi(x^{(k)}, \epsilon^{(k)})\right)$\\
     \hspace{3cm}
     $ + \log p_\theta \left(x^{(k)} \mid z_\phi(x^{(k)}, \epsilon^{(k)}) \right)\bigr]$

     \STATE Compute $\psi$-gradient  (eq. \ref{eq:discriminator-obj}) :\\
     \vspace{1ex}$
     g_\psi \gets  \frac{1}{m} \sum_{k=1}^m 
     \nabla_\psi \Bigl[
     \log \left(\sigma (T_\psi(x^{(k)}, z_\phi(x^{(k)}, \epsilon^{(k)})))\right)$\\
     \hspace{3cm}
     $+ \log \left(1 - \sigma(T_\psi(x^{(k)}, z^{(k)}) \right)
     \Bigr]$
     \vspace{1ex}
     
     \STATE Perform SGD-updates for $\theta$, $\phi$ and $\psi$:\\
     $\theta \gets \theta + h_i\,g_\theta, \quad
     \phi \gets \phi + h_i\,g_\phi, \quad
     \psi \gets \psi + h_i\,g_\psi$
     \STATE $i \gets i+1$
   \ENDWHILE
\end{algorithmic}
\end{algorithm}

In theory, Propositions \ref{prop:optimal-discriminator} and \ref{prop:gradient-discriminator-zero} allow
us to apply Stochastic Gradient Descent (SGD) directly to the objective in \eqref{eq:VAE-objective}.
However, this requires keeping $T^*(x,z)$ optimal which is computationally challenging. We therefore regard the
optimization problems in \eqref{eq:discriminator-obj} and \eqref{eq:avb-objective-reparameterization}
as a two-player game. Propositions \ref{prop:optimal-discriminator} and \ref{prop:gradient-discriminator-zero} show that any Nash-equilibrium 
of this game yields a stationary point of the objective in \eqref{eq:VAE-objective}.

In practice, we try to find a Nash-equilibrium by applying SGD with step sizes $h_i$ jointly to 
\eqref{eq:discriminator-obj} and \eqref{eq:avb-objective-reparameterization}, see Algorithm \ref{alg:avb}. Here, we parameterize the neural network $T$ with 
a vector $\psi$. Even though we have no guarantees that this
algorithm converges, any fix point of this algorithm yields a stationary point of the objective in \eqref{eq:VAE-objective}.

Note that optimizing \eqref{eq:avb-objective} with respect to $\phi$ while keeping $\theta$ and $T$ fixed makes the 
encoder network collapse to a deterministic function. This is also a common problem for regular GANs \cite{radford2015unsupervised}. It is therefore crucial to
keep the discriminative $T$ network close to optimality while optimizing \eqref{eq:avb-objective}. A variant of Algorithm \ref{alg:avb} therefore performs
several SGD-updates for the adversary for one SGD-update of the generative model.
However, throughout our experiments we use the simple $1$-step version of AVB unless stated otherwise.

\subsection{Theoretical results}\label{sec:theory}

In Sections \ref{sec:avb-derivation} we derived AVB as a way of performing stochastic gradient descent on the
variational lower bound in \eqref{eq:VAE-objective}. In this section, we analyze the properties of Algorithm \ref{alg:avb} from a game theoretical point of view.

As the next proposition shows, global Nash-equilibria of Algorithm \ref{alg:avb} yield global optima  of the objective in \eqref{eq:VAE-objective}:
\begin{restatable}{proposition}{probnashequilibriumoptimal}\label{prop:nash-equilibrium-optimal}
 Assume that $T$ can represent
  any function of two variables. If $(\theta^*, \phi^*, T^*)$ defines a Nash-equilibrium of the two-player game defined by \eqref{eq:discriminator-obj} and \eqref{eq:avb-objective-reparameterization},

 then
 \begin{equation}\label{eq:nonparameteric-thm-optimal-T}
    T^*(x, z) = \log q_{\phi^*}(z \mid x) - \log p(z)
 \end{equation}
 and $(\theta^*, \phi^*)$ is a global optimum of the variational lower bound in \eqref{eq:VAE-objective}.
\end{restatable}
\begin{proof}
 The proof can be found in the Supplementary Material.
\end{proof}
\vspace{-0.4cm}
Our parameterization of  $q_\phi(z\mid x)$ as a neural network allows $q_\phi(z\mid x)$ to represent almost any
probability density on the latent space. This motivates

\begin{corollary}\label{cor:nonparameteric-optimality}
 Assume that $T$ can represent
  any function of two variables and $q_\phi(z\mid x)$ can represent any
  probability density on the latent space.
  If $(\theta^*, \phi^*, T^*)$ defines 
 a Nash-equilibrium for the game defined by \eqref{eq:discriminator-obj} and \eqref{eq:avb-objective-reparameterization}, then
 \begin{enumerate}
  \item\label{enum:nonparameteric-optimality-cond-mle} $\theta^*$ is a maximum-likelihood assignment
  \item\label{enum:nonparameteric-optimality-cond-posterior} $q_{\phi^*}(z \mid x)$ is equal to the true posterior  $p_{\theta^*}(z \mid x)$
  \item\label{enum:nonparameteric-optimality-cond-T} $T^*$ is the pointwise mutual information between $x$ and $z$, i.e.
  \begin{equation}\label{eq:nonparameteric-optimal-T}
    T^*(x, z)=\log\frac{p_{\theta^*}(x,z)}{p_{\theta^*}(x)p(z)}.
  \end{equation}
 \end{enumerate}
\end{corollary}

\begin{proof}
 This is a straightforward consequence of Proposition \ref{prop:nash-equilibrium-optimal}, as in this case $(\theta^*, \phi^*)$
 optimizes the variational lower bound in \eqref{eq:VAE-objective} if and only if \ref{enum:nonparameteric-optimality-cond-mle} and
 \ref{enum:nonparameteric-optimality-cond-posterior} hold. Inserting the result from \ref{enum:nonparameteric-optimality-cond-posterior} into \eqref{eq:nonparameteric-thm-optimal-T}
 yields \ref{enum:nonparameteric-optimality-cond-T}.
 \end{proof}

\section{Adaptive Contrast}\label{sec:adaptive-contrast}

While in the nonparametric limit our method yields the correct results, in practice $T(x, z)$ may fail to
become sufficiently close to the optimal function $T^*(x,z)$. The reason for this problem is that AVB calculates a contrast between the
two densities $p_{\mathcal D}(x)q_\phi(z\mid x)$ to $p_{\mathcal D}(x) p(z)$ which are usually very different.
However, it is known that logistic
regression works best for likelihood-ratio estimation when comparing two very similar densities \cite{friedman2001elements}.

To improve the quality of the estimate, we therefore propose to introduce an auxiliary conditional probability distribution
$r_\alpha(z \mid x)$ with known density that approximates $q_\phi(z \mid x)$. For example, $r_\alpha(z \mid x)$ could be a Gaussian distribution
with diagonal covariance matrix whose mean and variance matches the mean and variance of $q_\phi(z\mid x)$.

Using this auxiliary distribution, we can rewrite the variational lower bound in \eqref{eq:VAE-objective} as
\begin{multline}\label{eq:VAE-objective-local-norm}
 \E_{p_{\mathcal D}(x)} \Bigl[
 -\KL\left(q_\phi(z \mid x), r_\alpha(z \mid x)\right) \\
 + \E_{q_\phi(z \mid x)} \left(-\log r_\alpha(z \mid x) + \log p_\phi(x , z)\right)
 \Bigr].
\end{multline}
As we know the density of $r_\alpha(z\mid x)$, the second term in \eqref{eq:VAE-objective-local-norm} is amenable to stochastic gradient descent with respect to $\theta$ and $\phi$.
However, we can estimate the first term using AVB as described in Section \ref{sec:avb}. 
If $r_\alpha(z \mid x)$ approximates $q_\phi(z\mid x)$ well, 
$\KL\left(q_\phi(z \mid x), r_\alpha(z \mid x)\right)$
is usually much smaller than
$\KL\left(q_\phi(z \mid x), p(z)\right)$, which makes it easier for the adversary to learn the correct probability ratio.

We call this technique Adaptive Contrast (AC),
as we are now contrasting the current inference model $q_\phi(z \mid x)$ to an adaptive distribution $r_\alpha(z \mid x)$ instead of the prior $p(z)$. 
Using Adaptive Contrast, the generative model $p_\theta(x \mid z)$ and the inference model $q_\phi(z \mid x)$ are trained to maximize
\begin{multline}
 \E_{p_{\mathcal{D}}(x)}\E_{q_{\phi}(z\mid x)}\big(-T^*(x, z) \\
 - \log r_\alpha(z\mid x)
 + \log p_{\theta}(x, z)\big),
\end{multline}
where $T^*(x,z)$ is the optimal discriminator distinguishing samples from $r_\alpha(z\mid x)$ and $q_\phi(z\mid x)$.

Consider now the case that $r_\alpha(z \mid x)$ is given by a Gaussian distribution
with diagonal covariance matrix whose mean $\mu(x)$ and variance vector $\sigma(x)$ match the mean and variance of $q_\phi(z\mid x)$.
As the Kullback-Leibler divergence is invariant under reparameterization, the first term in \eqref{eq:VAE-objective-local-norm} can be rewritten as
\begin{equation}\label{eq:VAE-objective-local-norm-Gauss-T}
 \E_{p_{\mathcal D}(x)} \KL\left(\tilde q_\phi(\tilde z \mid x), r_0(\tilde z)\right)
\end{equation}
where $\tilde q_\phi(\tilde z \mid x)$ denotes the distribution of the normalized vector $\tilde z := \tfrac{z - \mu(x)}{\sigma (x)}$ and $r_0(\tilde z)$ is a Gaussian distribution with mean $0$ and
variance $1$.
This way, the adversary only has to account for the deviation of $q_\phi(z \mid x)$ from a Gaussian distribution, not its location and scale. Please see the Supplementary Material for pseudo code of the resulting algorithm.

 In practice, we estimate $\mu(x)$ and
$\sigma(x)$ using a Monte-Carlo estimate. In the Supplementary Material we describe a network architecture for $q_\phi(z\mid x)$ that makes the computation of this
estimate particularly efficient.

\section{Experiments}
We tested our method both as a black-box method for variational inference and for learning generative models.
The former application corresponds to the case where we fix the generative model and a data point $x$ and want to learn the posterior $q_\phi(z \mid x)$.

An additional experiment on the celebA dataset \cite{liu2015faceattributes} can be found in the Supplementary Material.
\subsection{Variational Inference}\label{sec:var-inference}
\begin{figure}
\centering{}
\includegraphics[width=0.8\linewidth]{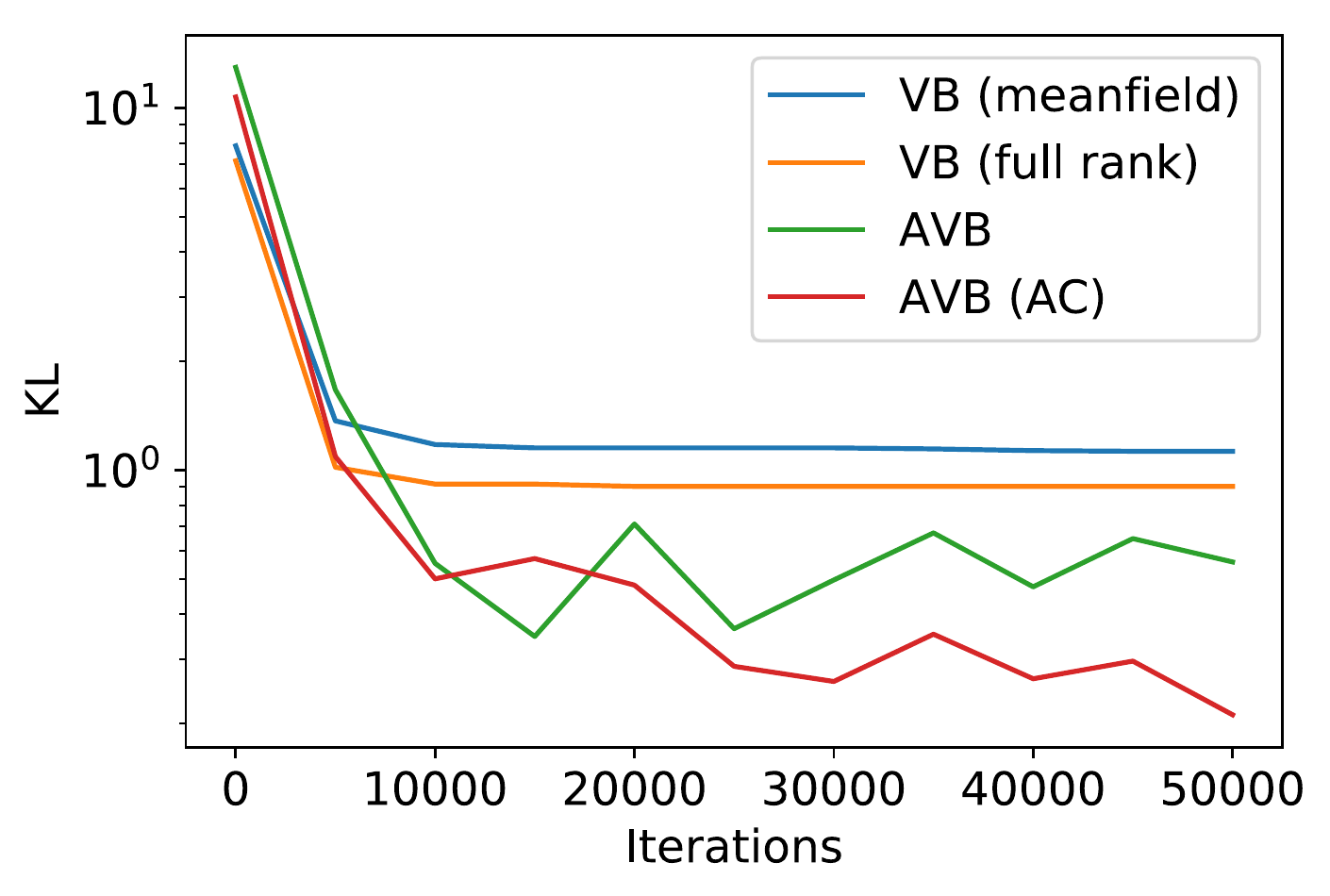}
 \vspace{-0.4cm}
\caption{\label{fig:pystan-comparison}Comparison of KL to ground truth posterior obtained by Hamiltonian Monte Carlo (HMC).}
 \vspace{-0.5cm}
\end{figure}

\begin{figure}[]
  \centering
  \begin{tabular}{C{2em}C{0.35\linewidth}C{0.35\linewidth}}
      & $(\mu, \tau)$ & $(\tau, \eta_1)$ \\
    AVB
    &\includegraphics[width=\linewidth]{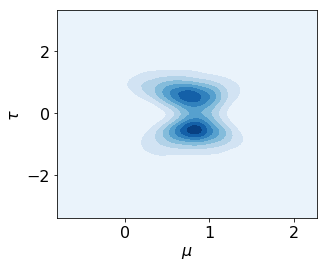}
    &\includegraphics[width=\linewidth]{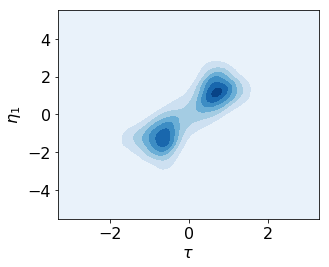}\\ 
    VB (fullrank)
    &\includegraphics[width=\linewidth]{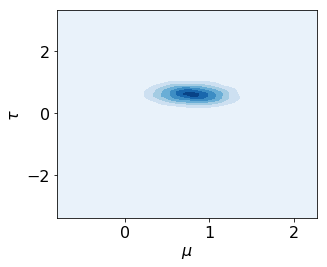}
    &\includegraphics[width=\linewidth]{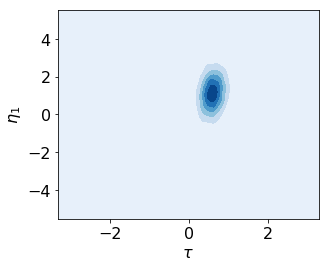}\\ 
    HMC
    &\includegraphics[width=\linewidth]{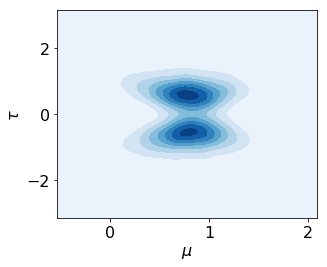}
    &\includegraphics[width=\linewidth]{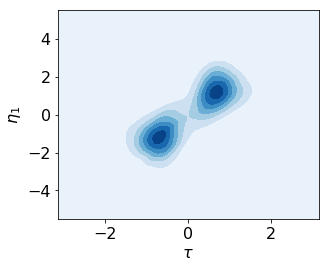}\\
   \end{tabular}
     \vspace{-0.2cm}

  \caption{\label{fig:pystan-marginals}Comparison of AVB to VB on the ``Eight Schools'' example by inspecting two marginal distributions of the approximation to the 10-dimensional posterior. We see that AVB accurately captures the multi-modality of the posterior distribution. In contrast, VB only focuses on a single mode. The ground truth is shown in the last row and has been obtained using HMC.}
  \vspace{-0.5cm}
\end{figure}
When the generative model and a data point $x$ is fixed, AVB gives a new technique for Variational Bayes with arbitrarily complex approximating distributions. We applied this to the ``Eight School'' example from
\citet{gelman2014bayesian}. In this example, the coaching effects $y_i$, $i=1,\dots,8$ for eight schools are modeled as
\begin{equation*}
 y_i \sim \mathcal N(\mu + \theta \cdot \eta_i, \sigma_i),
\end{equation*}
where $\mu$, $\tau$ and the $\eta_i$ are the model parameters to be inferred. We place a $\mathcal N(0, 1)$ prior on the parameters of the model.
We compare AVB against two variational methods with Gaussian inference model \cite{kucukelbir2015automatic} as implemented in STAN \cite{stan_development_team_stan_2016}. We used a simple two layer model for the posterior and a powerful $5$-layer network with RESNET-blocks \cite{he2015deep} for the discriminator. For every posterior update step we performed two steps for the adversary. The ground-truth data was obtained by running Hamiltonian Monte-Carlo (HMC) for 500000 steps using STAN.
Note that AVB and the baseline variational methods allow to draw an arbitrary number of samples after training is completed whereas HMC only yields a fixed number of samples. 

We evaluate all methods by estimating the Kullback-Leibler-Divergence to the ground-truth data using the ITE-package \cite{szabo2013information} applied
to $10000$ samples from the ground-truth data and the respective approximation.
The resulting Kullback-Leibler divergence over the number of iterations for the different methods is plotted in Figure \ref{fig:pystan-comparison}. We see that
our method clearly outperforms the methods with Gaussian inference model. For a qualitative visualization, we also applied Kernel-density-estimation to the 2-dimensional marginals of the $(\mu, \tau)$- and $(\tau, \eta_1)$-variables as illustrated in Figure \ref{fig:pystan-marginals}. In contrast to variational Bayes with Gaussian inference model, our approach clearly captures the multi-modality of the posterior distribution.
We also observed that Adaptive Contrast makes learning more robust and improves the quality of the resulting model.

\FloatBarrier
\subsection{Generative Models}
\paragraph{Synthetic Example}
\begin{figure}[t]
\centering
{\includegraphics[width=0.1\linewidth]{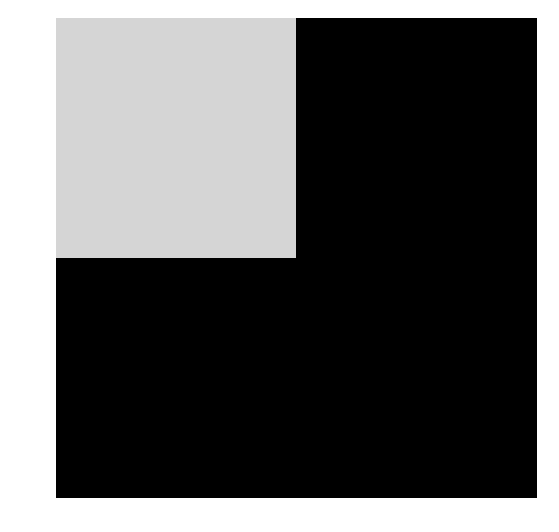}}
{\includegraphics[width=0.1\linewidth]{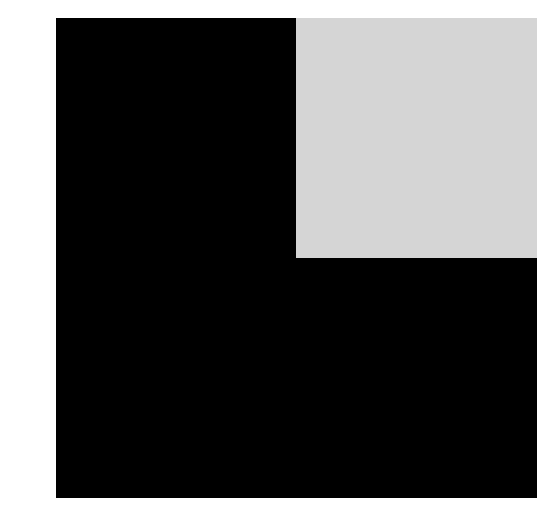}}
{\includegraphics[width=0.1\linewidth]{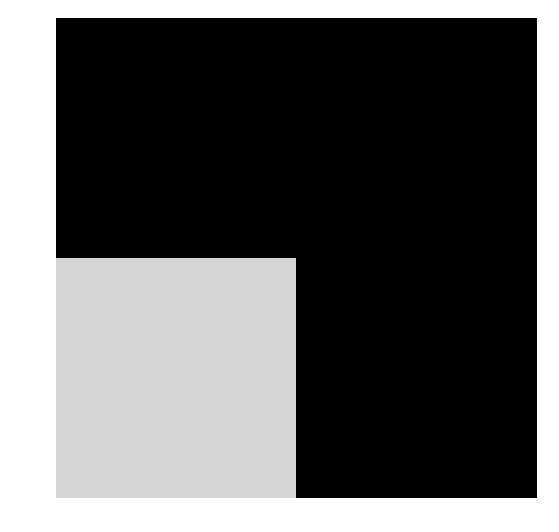}}
{\includegraphics[width=0.1\linewidth]{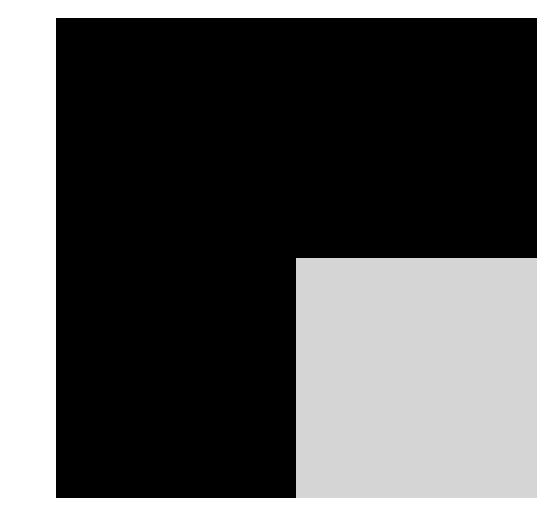}}
\caption{\label{fig:synthetic-samples}Training examples in the synthetic 
dataset.}
\vspace{0.5cm}
\centering
\begin{subfigure}[b]{0.35\linewidth}
\centering
\includegraphics[width=\linewidth]{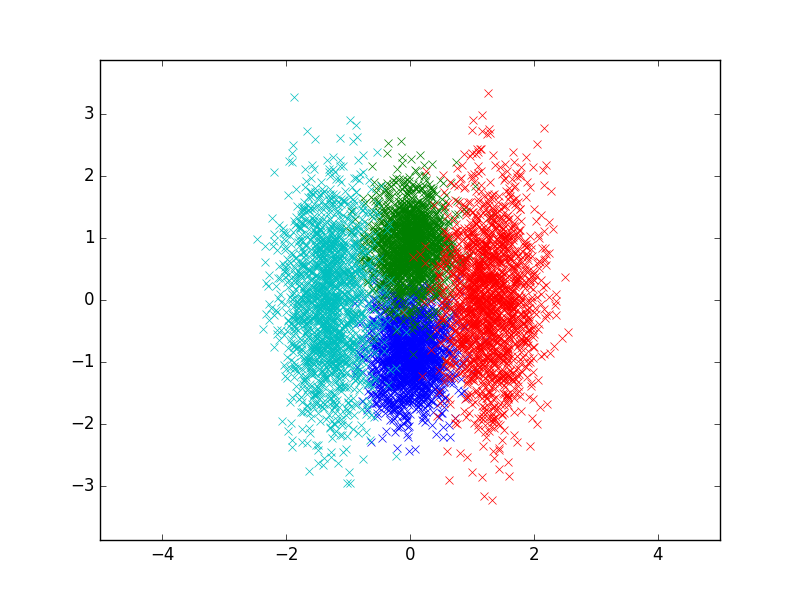}
\caption{VAE}
\end{subfigure}
\begin{subfigure}[b]{0.35\linewidth}
\centering
\includegraphics[width=\linewidth]{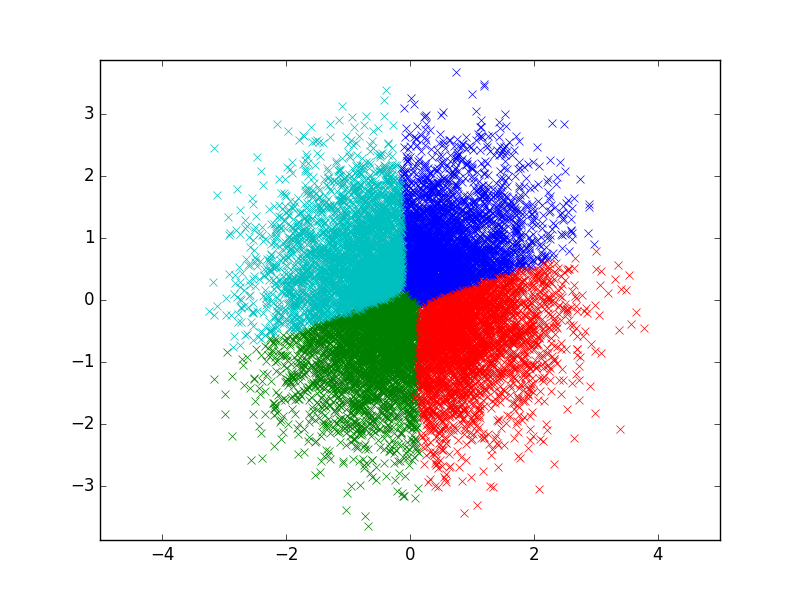}
\caption{AVB}
\end{subfigure}

\caption{\label{fig:latent}Distribution of latent code for VAE and AVB trained
on synthetic dataset.}
\vspace{-0.4cm}
\end{figure}

\begin{table}[t]
\centering
\resizebox{0.66\linewidth}{!}{
\begin{tabular}{ l | c c }
      & VAE &  AVB \\
  \hline 	
  log-likelihood & -1.568 & \bf -1.403 \\
  reconstruction error & 88.5   $\cdot 10^{-3}$ &  5.77 $\cdot 10^{-3}$ \\
  ELBO & -1.697 &  $\approx$ -1.421 \\
  $\KL(q_\phi(z), p(z))$ & $\approx$ 0.165  & $\approx$ \bf 0.026
\end{tabular}
}
\caption{\label{tab:latent}Comparison of VAE and AVB on synthetic dataset. The 
optimal log-likelihood score
on this dataset is $-\log(4) \approx -1.386$.}
\vspace{-0.4cm}
\end{table}

To illustrate the application of our method to 
learning a generative model, we trained the neural networks on a simple 
synthetic dataset containing only the
$4$ data points from the space of $2\times 2$ binary images shown in Figure \ref{fig:synthetic-samples} and a 
$2$-dimensional latent space.
Both the encoder and decoder are parameterized by 
2-layer fully 
connected neural networks with 512 hidden units each. The encoder network takes 
as input a data point $x$ and a vector of Gaussian random noise $\epsilon$ and 
produces
a latent code $z$. The decoder network takes as input a latent code $z$ and 
produces the parameters for four independent Bernoulli-distributions, one for 
each pixel of the output image.
The adversary is parameterized by two neural networks with two $512$-dimensional 
hidden layers each,
acting on $x$ and $z$ respectively, whose $512$-dimensional outputs are 
combined using an inner product.

We compare our method to a Variational Autoencoder with a diagonal Gaussian 
posterior distribution. The encoder and decoder networks are parameterized 
as above, but the encoder does not take the noise $\epsilon$ 
as input and produces a mean and variance vector instead of a single sample.

We visualize the 
resulting division of the latent space in Figure \ref{fig:latent}, where each 
color corresponds to 
one state in the $x$-space. Whereas the Variational Autoencoder divides the 
space
into a mixture of $4$ Gaussians, the Adversarial Variational Autoencoder learns 
a complex posterior distribution. Quantitatively this can be verified by 
computing
the KL-divergence between the prior $p(z)$ and the aggregated posterior 
$q_\phi(z):=\int q_\phi(z \mid x) p_\mathcal D(x) \mathrm d x$, which we estimate
using the ITE-package \cite{szabo2013information}, see Table \ref{tab:latent}.
Note that the variations for different colors in Figure \ref{fig:latent} are solely due to the noise 
$\epsilon$ used in the inference model.

The ability of AVB to learn more complex posterior models leads to improved 
performance as Table \ref{tab:latent} shows.
In particular, AVB leads to a higher likelihood score that is close to the 
optimal value of $-\log(4)$ compared to a standard VAE that struggles with the 
fact that it cannot divide the
latent space appropriately. Moreover, we see that the reconstruction error given 
by the mean cross-entropy between an input $x$ and its reconstruction using the 
encoder and decoder networks is much lower when using AVB instead of a VAE with 
diagonal Gaussian inference model. 
We also observe that the estimated variational lower bound is close to the true 
log-likelihood, indicating that the adversary has learned the correct function.

\paragraph{MNIST}
\begin{figure}[t]
\centering
\begin{subfigure}[b]{0.4\linewidth}
\centering
\includegraphics[width=\linewidth]{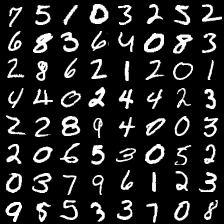}
\caption{Training data}
\end{subfigure}
\begin{subfigure}[b]{0.4\linewidth}
\centering
\includegraphics[width=\linewidth]{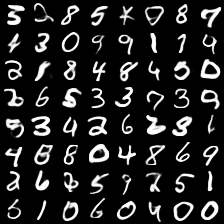}
\caption{Random samples}
\end{subfigure}
\caption{\label{fig:mnist-samples}Independent samples for a model trained
on MNIST}
\vspace{-0.4cm}
\end{figure}

\begin{table}[t]
\centering
\resizebox{\linewidth}{!}{
 \begin{tabular}{ l | c c r}
 
   & $\log p(x) \geq$ & $\log p(x) \approx$ &  \\
  \hline
  AVB (8-dim)  & $(\approx -83.6 \pm 0.4)$ & $-91.2 \pm 0.6$ & \\
  AVB + AC (8-dim)  & $\approx -96.3 \pm 0.4$ & $-89.6 \pm 0.6$ & \\
  AVB + AC (32-dim) & $\approx -79.5 \pm 0.3$   & $-80.2 \pm 0.4$ & \\
  VAE (8-dim) &  $-98.1 \pm 0.5$ & $-90.9 \pm 0.6$ & \\
  VAE (32-dim)&  $-87.2 \pm 0.3$ & $-81.9 \pm 0.4$ & \\
  \hline
  VAE + NF (T=80)  & $-85.1$ & $-$ & \small\cite{rezende2015variational} \\
  VAE + HVI  (T=16)  & $-88.3$ & $-85.5 $& \small\cite{salimans2015markov}\\
  convVAE + HVI (T=16)  & $-84.1$ & $-81.9$& \small\cite{salimans2015markov}\\
  VAE + VGP (2hl)& $-81.3$ & $-$& \small\cite{tran2015variational}\\
  DRAW + VGP & $-79.9$ &$-$ & \small\cite{tran2015variational} \\
  VAE + IAF & $-80.8$ & $-79.1$ & \small\cite{kingma2016improving} \\
  Auxiliary VAE (L=2) &$-83.0$ & $-$ & \small\cite{maaloe2016auxiliary}
\end{tabular}
}
\vspace{-0.2cm}

\caption{\label{tab:mnist}
Log-likelihoods on binarized MNIST for AVB and other methods improving on VAEs.
We see that our method achieves state of the art log-likelihoods on binarized MNIST. The approximate log-likelihoods in the lower half of the table were not obtained with AIS but with importance sampling.
}
\vspace{-0.5cm}
\end{table}

In addition, we trained deep convolutional networks based on the 
DC-GAN-architecture \cite{radford2015unsupervised} on the binarized
MNIST-dataset \cite{lecun1998gradient}. For the decoder network, we use a $5$-layer deep convolutional neural network. For the encoder network, we use a network architecture
that allows for the efficient computation of the moments of $q_\phi(z \mid x)$. 
The idea is to define the encoder as a linear combination of learned basis noise vectors, each parameterized by a small fully-connected neural network,
whose coefficients are parameterized by a neural network acting on $x$, please see the Supplementary Material for details.
For 
the adversary, we replace the fully connected neural network
acting on $z$ and $x$ with a fully connected $4$-layer neural networks with $1024$ units in each hidden layer.  In addition, we added the result of neural networks acting on $x$ and $z$ alone to the end result.

To validate our method, we ran Annealed Importance Sampling (AIS) \cite{neal2001annealed}, the gold standard for evaluating decoder based generative models \cite{wu2016quantitative} with $1000$ intermediate 
distributions and $5$ parallel chains on $2048$ test examples. The results are reported in Table \ref{tab:mnist}. 
Using AIS, we see that AVB without AC overestimates 
the true ELBO which degrades its performance.
Even though the results suggest that AVB with AC can also overestimate the true ELBO in higher dimensions,
we note that the log-likelihood estimate computed by AIS is also only a lower bound to the true log-likelihood \cite{wu2016quantitative}. 
 
Using AVB with AC, we see that we improve both on a standard VAE and AVB without AC.
When comparing to other state of the art methods, we see that our method achieves state of the art results on binarized MNIST\footnote{Note that the methods in the lower half of Table~\ref{tab:mnist} were trained with different decoder architectures and therefore only provide limited information regarding the quality of the inference model.}. For an additional experimental evaluation of AVB and  three baselines for a fixed decoder architecture see the Supplementary Material. 
Some random samples for MNIST are shown in Figure \ref{fig:mnist-samples}. We 
see that our model produces random samples that are perceptually close to the training 
set.

\section{Related Work}

\subsection{Connection to Variational Autoencoders}
AVB strives to optimize the same objective as a standard VAE  \cite{kingma2013auto,rezende2014stochastic},
but approximates
the Kullback-Leibler divergence using an adversary instead of relying on a closed-form formula.

Substantial work has focused on making the class of approximate
inference models more expressive. Normalizing flows \cite{rezende2015variational, kingma2016improving} make the posterior more complex by composing a simple Gaussian posterior
with an invertible smooth mapping for which the determinant of the Jacobian is tractable. Auxiliary Variable VAEs \cite{maaloe2016auxiliary} add auxiliary variables to the posterior to
make it more flexible.
However, no other approach that we are aware of allows to use black-box
inference models to optimize the ELBO.

\subsection{Connection to Adversarial Autoencoders}\label{sec:related-aae}

Makhzani et al. \cite{makhzani2015adversarial} introduced the concept
of Adversarial Autoencoders. The idea is to replace the term
\begin{equation}
 \KL(q_\phi(z \mid x), p(z))
\end{equation}
in \eqref{eq:VAE-objective} with an adversarial loss that tries to enforce that upon convergence
\begin{equation}
 \int q_\phi(z \mid x) p_{\mathcal D}(x) \mathrm d x \approx p(z).
\end{equation}
While related to our approach, the
approach by Makhzani et al. modifies the variational objective while our approach retains
the objective. 

The approach by Makhzani et al. can
be regarded as an approximation to our approach, where $T(x,z)$ is
restricted to the class of functions that do not depend on $x$. Indeed,
an ideal discriminator that only depends on $z$ maximizes
\begin{multline}
\int\int p_{\mathcal{D}}(x)q(z\mid x)\log\sigma(T(z))\mathrm{d}x\mathrm{d}z\\
+\int\int p_{\mathcal{D}}(x)p(z)\log\left(1-\sigma(T(z)\right))\mathrm{d}x\mathrm{d}z
\end{multline}
which is the case if and only if
\begin{equation}
T(z)=\log\int q(z\mid x)p_{\mathcal{D}}(x)\mathrm{d}x-\log p(z).
\end{equation}
Clearly, this simplification is a crude approximation to our formulation from Section \ref{sec:avb}, but \citet{makhzani2015adversarial}
show that this method can still lead to good sampling results. In theory, restricting
$T(x,z)$ in this way ensures that upon convergence we approximately have
\begin{equation}
\int q_\phi(z\mid x)p_{\mathcal{D}}(x)\mathrm{d}x=p(z),
\end{equation}
but $q_\phi(z\mid x)$ need not be close to the true posterior $p_\theta(z\mid x)$.
Intuitively, while mapping $p_{\mathcal{D}}(x)$ through $q_\phi(z\mid x)$
results in the correct marginal distribution, the contribution of
each $x$ to this distribution can be very inaccurate.

In contrast to Adversarial Autoencoders, our goal is to improve the ELBO by performing better probabilistic inference. This allows our method to be used in a more general setting where we are only interested in the inference network itself (Section~\ref{sec:var-inference}) and enables further improvements such as Adaptive Contrast (Section~\ref{sec:adaptive-contrast}) which are not possible in the context of Adversarial Autoencoders. 

\subsection{Connection to f-GANs}

Nowozin et al. \cite{nowozin2016f} proposed to generalize Generative
Adversarial Networks \cite{goodfellow2014generative} to f-divergences
\cite{ali1966general} based on results by Nguyen et al. \cite{nguyen2010estimating}.
In this paragraph we show that f-divergences allow to represent
AVB as a zero-sum two-player game.

The family of f-divergences is given by
\begin{equation}
D_{f}(p\|q)
=\E_{p}f\left(\frac{q(x)}{p(x)}\right).
\end{equation}
for some convex functional $f:\mathbb{R}\to\mathbb{R}_{\infty}$ with
$f(1)=0$.

\citet{nguyen2010estimating} show that by using the convex conjugate $f^*$ of
$f$, \cite{hiriart2013convex}, we obtain
\begin{equation}
D_{f}(p\|q)
=  \sup_{T}\E_{q(x)}\left[T(x)\right] -  \E_{p(x)}\left[f^{*}(T(x))\right],
\end{equation}
where $T$ is a real-valued function. In particular, this is true
for the reverse Kullback-Leibler divergence with $f(t)=t\log t$.
We therefore obtain
\begin{multline}
\KL(q(z\mid x),p(z))=D_{f}(p(z),q(z\mid x))\\
=\sup_{T}\E_{q(z\mid x)}T(x,z)-\E_{p(z)}f^{*}(T(x,z)),
\end{multline}
with $f^{*}(\xi)=\exp(\xi-1)$ the convex conjugate of $f(t)=t\log t$.

All in all, this yields
\begin{eqnarray}
& & \max_{\theta} \E_{p_{\mathcal{D}}(x)}\log p_\theta(x)\label{eq:adversary-objective}\\
& = & \max_{\theta,q} \min_{T} \E_{p_{D}(x)}\E_{p(z)} f^{*}(T(x,z))\nonumber\\
& & \qquad +\E_{p_{D}(x)}\E_{q(z\mid x)}(\log p_{\theta}(x\mid
	z)-T(x,z)).\nonumber
\end{eqnarray}
By replacing the objective \eqref{eq:discriminator-obj} for the discriminator with
\begin{equation}
\min_{T}\E_{p_{D}(x)} \left[\E_{p(z)}\mathrm{e}^{T(x,z)-1}
-\E_{q(z\mid x)}T(x,z) \right],\label{eq:alternative-discriminator-obj}
\end{equation}
we can reformulate the maximum-likelihood-problem as a mini-max zero-sum game. In fact,
the derivations from Section \ref{sec:avb} remain valid for any $f$ -divergence that we use to train the discriminator. 
This is similar to the approach taken by Poole et al. \cite{poole2016improved}
to improve the GAN-objective. In practice, we observed that the objective
(\ref{eq:alternative-discriminator-obj}) results in unstable training.
We therefore used the standard GAN-objective (\ref{eq:discriminator-obj}), which corresponds to the Jensen-Shannon-divergence.

\subsection{Connection to BiGANs}
BiGANs \cite{donahue2016adversarial,dumoulin2016adversarially} are a recent extension to Generative Adversarial Networks with the goal
to add an inference network to the generative model. Similarly to our approach, the authors introduce an adversary that acts on pairs $(x, z)$ of data points and
latent codes. However, whereas in BiGANs the adversary is used to optimize the generative and inference networks separately, our approach optimizes
the generative and inference model jointly. As a result, our approach obtains good reconstructions of the input data, whereas for BiGANs
we obtain these reconstructions only indirectly.

\section{Conclusion}
We presented a new training procedure for Variational Autoencoders
based on adversarial training. This allows us to make the inference
model much more flexible, effectively allowing it to represent almost
any family of conditional distributions over the latent variables.

We believe that further progress can be made by investigating the
class of neural network architectures used for the adversary and the
encoder and decoder networks as well as finding better contrasting distributions. 

\section*{Acknowledgements}
This work was supported by Microsoft Research through its PhD Scholarship Programme.

\bibliographystyle{icml2017}
\bibliography{bib/bibliography}

\twocolumn[
\icmltitle{Supplementary Material for \\ Adversarial Variational Bayes: Unifying Variational Autoencoders and Generative Adversarial Networks}
\begin{icmlauthorlist}
\icmlauthor{Lars Mescheder}{avg}\qquad
\icmlauthor{Sebastian Nowozin}{msr}\qquad
\icmlauthor{Andreas Geiger}{avg,cvg}
\end{icmlauthorlist}
\icmlaffiliation{avg}{Autonomous Vision Group, MPI T\"ubingen}
\icmlaffiliation{msr}{Microsoft Research Cambridge}
\icmlaffiliation{cvg}{Computer Vision and Geometry Group, ETH Z\"urich}
\icmlcorrespondingauthor{Lars Mescheder}{lars.mescheder@tuebingen.mpg.de}
\icmlkeywords{machine learning, GAN, VAE, generative models}
\vskip 0.3in
]
\renewcommand{\thesection}{\Roman{section}}
\setcounter{section}{0}
\begin{abstract}
 In the main text we derived Adversarial Variational Bayes (AVB) and
 demonstrated its usefulness both for black-box Variational Inference and
 for learning latent variable models. This document contains proofs that
 were omitted in the main text as well as some further details about the experiments
 and additional results.
\end{abstract}

\section{Proofs}
This section contains the proofs that were omitted in the main text.

The derivation of AVB in Section \ref{sec:avb-derivation} relies on the fact that we have an explicit representation of the optimal discriminator $T^*(x,z)$. This was stated in the 
following Proposition:

\propoptimaldiscriminator*

\begin{proof}
As in the proof of Proposition 1 in \citet{goodfellow2014generative}, we rewrite the objective in \eqref{eq:discriminator-obj} as
\begin{multline}
 \int \bigl( p_{\mathcal D}(x) q_\phi(z \mid x) \log \sigma (T(x, z)) \\
 +  p_{\mathcal D}(x) p(z) \log(1 - \sigma(T(x, z)) \bigr) \mathrm d x \mathrm d z.
\end{multline}
This integral is maximal as a function of $T(x,z)$ if and only if the integrand is maximal for every $(x,z)$. However, the function
\begin{equation}
 t \mapsto a \log(t) + b\log(1 - t) 
\end{equation}
attains its maximum at $t = \tfrac{a}{a + b}$, showing that
\begin{equation}
 \sigma( T^*(x, z))
 = \frac{q_\phi(z \mid x)}{q_\phi(z \mid x) + p(z)}
\end{equation}
or, equivalently,
\begin{equation}
  T^*(x,z) = \log q_\phi(z \mid x) - \log p(z).
\end{equation}
\end{proof}

To apply our method in practice, we need to obtain unbiased gradients of the ELBO. As it turns out, this can be achieved by taking the gradients w.r.t. a fixed optimal discriminator. This is a consequence of the following Proposition:
\propgradientdiscriminatorzero*
\begin{proof}
 By Proposition \ref{prop:optimal-discriminator},
 \begin{multline}\label{eq:gradient-discriminator}
   \E_{q_{\phi}(z\mid x)}\left(\nabla_\phi T^*(x, z)\right) \\
   = \E_{q_{\phi}(z\mid x)}\left(\nabla_\phi \log q_{\phi}(z\mid x) \right).
 \end{multline}
  For an arbitrary family of probability densities $q_\phi$ we have
  \begin{multline}
    \E_{q_\phi} \left( \nabla_\phi \log q_\phi \right)
    = \int q_\phi(z) \frac{\nabla_\phi q_\phi(z)}{q_\phi(z)} \mathrm d z \\
    = \nabla_\phi  \int q_\phi(z) \mathrm d z = \nabla_\phi  1 = 0.
  \end{multline}
  Together with \eqref{eq:gradient-discriminator}, this implies \eqref{eq:gradient-discriminator-zero}.
\end{proof}

In Section \ref{sec:theory} we characterized the Nash-equilibria of the two-player game defined by our algorithm. The following Proposition shows that in the nonparametric limit
for  $T(x,z)$ any Nash-equilibrium defines a global optimum of the variational lower bound:

\probnashequilibriumoptimal*
\begin{proof}
If $(\theta^*, \phi^*, T^*)$ defines a Nash-equilibrium, Proposition \ref{prop:optimal-discriminator} shows \eqref{eq:nonparameteric-thm-optimal-T}.
Inserting \eqref{eq:nonparameteric-thm-optimal-T} into \eqref{eq:avb-objective} shows that $(\phi^*, \theta^*)$ maximizes
\begin{multline}\label{eq:nonparameteric-thm-avb-objective}
 \E_{p_{\mathcal{D}}(x)}\E_{q_{\phi}(z\mid x)} \bigl(
  -\log q_{\phi^*}(z \mid x)  + \log p(z) \\
  + \log p_\theta(x \mid z)
  \bigr)
\end{multline}
as a function of $\phi$ and $\theta$. 
A straightforward calculation shows that \eqref{eq:nonparameteric-thm-avb-objective}
is equal to
\begin{equation}\label{eq:nonparameteric-thm-avb-objective-2}
 \mathcal L(\theta, \phi)
 +  \E_{p_{\mathcal{D}}(x)} \KL(q_\phi(z \mid x), q_{\phi^*}(z \mid x))
\end{equation}
where
\begin{multline}
 \mathcal L(\theta, \phi)
 :=  \E_{p_{\mathcal{D}}(x)} \Big[ -\KL(q_\phi(z \mid x), p(z)) \\
    + \E_{q_\phi(z \mid x)} \log p_\theta (x \mid z) \Big] 
\end{multline}
is the variational lower bound in \eqref{eq:VAE-objective}.

Notice that \eqref{eq:nonparameteric-thm-avb-objective-2} evaluates to $\mathcal L(\theta^*, \phi^*)$ when we insert  $(\theta^*, \phi^*)$ 
for $(\theta, \phi)$.

Assume now, that $(\theta^*, \phi^*)$ does not maximize the variational lower bound
$\mathcal L(\theta, \phi)$.
Then there is  $(\theta', \phi')$ with
\begin{equation}
 \mathcal L(\theta', \phi') > \mathcal  L(\theta^*, \phi^*).
\end{equation}

Inserting $(\theta', \phi')$ for
$(\theta, \phi)$ in \eqref{eq:nonparameteric-thm-avb-objective-2} we obtain
\begin{equation}
\mathcal L(\theta', \phi') +  \E_{p_{\mathcal{D}}(x)} \KL(q_{\phi'}(z \mid x), q_{\phi^*}(z \mid x)),
\end{equation}
which is strictly bigger than $\mathcal L(\theta^*, \phi^*)$, contradicting the fact that $(\theta^*, \phi^*)$ maximizes
\eqref{eq:nonparameteric-thm-avb-objective-2}.
Together with \eqref{eq:nonparameteric-thm-optimal-T}, this proves the theorem.
\end{proof}

\vfill
\section{Adaptive Contrast}
In Section~\ref{sec:adaptive-contrast} we derived a variant of AVB that contrasts the current inference model with an adaptive distribution rather than the prior. This leads to Algorithm~\ref{alg:avb-ac}. Note that we do not consider the $\mu^{(k)}$ and $\sigma^{(k)}$ to be functions of $\phi$ and therefore  do not backpropagate gradients through them.

\begin{algorithm}[H]
\caption{Adversarial Variational Bayes with Adaptive Constrast (AC)}
\label{alg:avb-ac}
\begin{algorithmic}[1]
    \STATE $i \gets 0$
   \WHILE{not converged}
     \STATE Sample $\{x^{(1)}, \dots, x^{(m)}\}$ from data distrib. $p_\mathcal D(x)$
     \STATE Sample $\{z^{(1)}, \dots, z^{(m)}\}$ from prior $p(z)$
     \STATE Sample $\{\epsilon^{(1)}, \dots, \epsilon^{(m)}\}$ from $\mathcal N(0, 1)$
     \STATE Sample $\{\eta^{(1)}, \dots, \eta^{(m)}\}$ from $\mathcal N(0, 1)$
     \FOR{$k = 1,\dots, m$}
       \STATE $z_\phi^{(k)}, \mu^{(k)}, \sigma^{(k)} \gets \mathrm{encoder}_\phi(x^{(k)}, \epsilon^{(k)})$
       \STATE $\bar z_\phi^{(k)} \gets \tfrac{z_\phi^{(k)} - \mu^{(k)}}{\sigma^{(k)}}$
     \ENDFOR
     \STATE Compute $\theta$-gradient (eq. \ref{eq:avb-objective-reparameterization}):\\
     \vspace{1ex}$
     g_\theta \gets \frac{1}{m}  \sum_{k=1}^m
     \nabla_\theta \log p_\theta \left(x^{(k)} , z_\phi^{(k)} \right)
     $\vspace{1ex}
     
     \STATE Compute $\phi$-gradient (eq. \ref{eq:avb-objective-reparameterization}):\\
     \vspace{1ex}$
     g_\phi \gets \frac{1}{m} \sum_{k=1}^m 
     \nabla_\phi \bigl[
     -T_\psi \left(x^{(k)}, \bar z_\phi^{(k)} \right) + \tfrac{1}{2}\|\bar z^{(k)}_\phi \|^2$\\
     \hspace{3cm}
     $ + \log p_\theta \left(x^{(k)} , z_\phi^{(k)} \right)\bigr]$

     \STATE Compute $\psi$-gradient  (eq. \ref{eq:discriminator-obj}) :\\
     \vspace{1ex}$
     g_\psi \gets \frac{1}{m} \sum_{k=1}^m 
     \nabla_\psi \Bigl[
     \log \left(\sigma (T_\psi(x^{(k)}, \bar z_\phi^{(k)}\right)$\\
     \hspace{3cm}
     $+ \log \left(1 - \sigma(T_\psi(x^{(k)}, \eta^{(k)}) \right)
     \Bigr]$
     \vspace{1ex}
     
     \STATE Perform SGD-updates for $\theta$, $\phi$ and $\psi$:\\
     $\theta \gets \theta + h_i\,g_\theta, \quad
     \phi \gets \phi + h_i\,g_\phi, \quad
     \psi \gets \psi + h_i\,g_\psi$
     \STATE $i \gets i+1$
   \ENDWHILE
\end{algorithmic}
\end{algorithm}

\section{Architecture for MNIST-experiment}
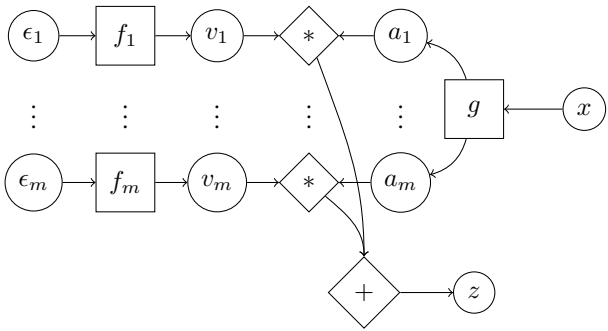
\begin{figure}[h]
\centering
\resizebox{\linewidth}{!}{
\begin{tikzpicture}[scale=0.5]

\tikzstyle{edge}=[]
\tikzstyle{vertex}=[circle,fill=white,draw, minimum width=1.1]
\tikzstyle{function}=[ draw, rectangle, minimum size=0.8cm]
\tikzstyle{operation}=[ draw, diamond,]


 \node[vertex] (eps1)  at (-5, 1) {$\epsilon_1$} ;
 \node (epsdots)  at (-5, -1) {$\vdots$} ;
 \node[vertex] (epsm)  at (-5, -3) {$\epsilon_m$} ;

\node[function](f1) at (-2.5,1) {$f_1$};
\node(fdots) at (-2.5,-1) {$\vdots$};
\node[function](fm) at (-2.5,-3) {$f_m$};

 \node[vertex] (v1)  at (0, 1) {$v_1$} ;
 \node (vdots)  at (0, -1) {$\vdots$} ;
 \node[vertex] (vm)  at (0, -3) {$v_m$} ;

 \node[vertex] (a1)  at (5, 1) {$a_1$} ;
 \node (vdots)  at (5, -1) {$\vdots$} ;
 \node[vertex] (am)  at (5, -3) {$a_m$} ;

\node[function](g) at (7, -1) {$g$};
 \node[vertex] (x)  at (10, -1) {$x$} ;

 \node[operation] (prod1)  at (2.5, 1) {$*$} ;
 \node (epsdots)  at (2.5, -1) {$\vdots$} ;
 \node[operation] (prodm)  at (2.5, -3) {$*$} ;

 \node[operation] (sum)  at (4, -6) {$+$} ;
 \node[vertex] (z)  at (7, -6) {$z$} ;

\draw[->, edge] (eps1) -- (f1);
\draw[->, edge] (f1) -- (v1);
\draw[->, edge] (v1) -- (prod1);

\draw[->, edge] (epsm) -- (fm);
\draw[->, edge] (fm) -- (vm);
\draw[->, edge] (vm) -- (prodm);

\draw[->, edge] (x) -- (g);

\draw[->, edge] (g) edge[bend right] (a1);
\draw[->, edge] (g) edge[bend left] (am);

\draw[->, edge] (a1) -- (prod1);
\draw[->, edge] (am) -- (prodm);

\draw[->, edge] (prod1) edge[ out=-70,in=90] (sum);
\draw[->, edge] (prodm) edge[out=-40,in=90] (sum);

\draw[->, edge] (sum) -- (z);

\end{tikzpicture}
}
\caption{\label{fig:network-mnist}Architecture of the network used for the MNIST-experiment}
\end{figure}
To apply Adaptive Contrast to our method, we have to be able to efficiently estimate the moments of the current inference model $q_\phi(z\mid x)$. 
To this end, we propose a network architecture like in Figure \ref{fig:network-mnist}. The final output $z$ of the network is a linear combination of basis noise vectors where the coefficients
depend on the data point $x$, i.e.
\begin{equation}
 z_k = \sum_{i=1}^m v_{i,k}(\epsilon_i)a_{i,k}(x).
\end{equation}
The noise basis vectors $v_i(\epsilon_i)$ are defined as the output of small fully-connected neural networks $f_i$ acting on normally-distributed random noise $\epsilon_i$, the coefficient vectors
$a_i(x)$ are defined as the output of a deep convolutional neural network $g$ acting on $x$.

The moments of the $z_i$ are then given by
\begin{align}
 \E(z_k) & = \sum_{i=1}^m \E [v_{i,k}(\epsilon_i)] a_{i,k}(x). \\
 \Var(z_k) & = \sum_{i=1}^m \Var [v_{i,k}(\epsilon_i)] a_{i,k}(x)^2. 
\end{align}
By estimating $\E [v_{i,k}(\epsilon_i)]$ and $\Var [v_{i,k}(\epsilon_i)]$ via sampling once per mini-batch, we can efficiently compute the moments of $q_\phi(z\mid x)$ for all the data points $x$ in a single mini-batch.

\section{Additional Experiments}

\paragraph{celebA}

\begin{figure}[t]
\centering
\begin{subfigure}[b]{0.45\linewidth}
\centering
\includegraphics[width=\linewidth]{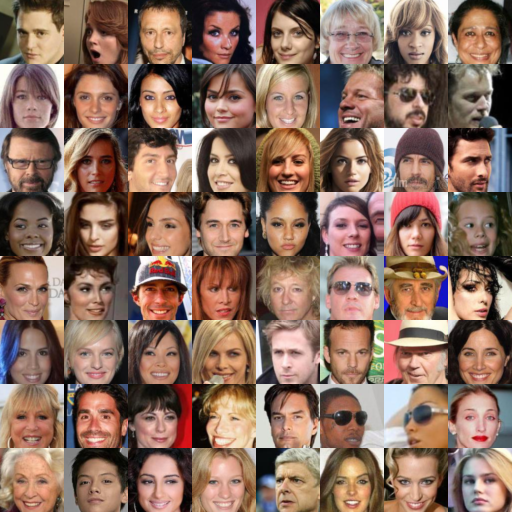}
\caption{Training data}
\end{subfigure}
\begin{subfigure}[b]{0.45\linewidth}
\centering
\includegraphics[width=\linewidth]{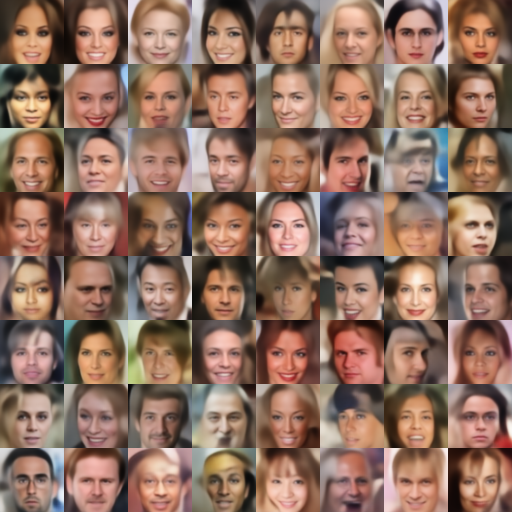}
\caption{Random samples}
\end{subfigure}
\caption{\label{fig:celebA-samples}Independent samples for a model trained
on celebA. 
}
\vspace{0.5cm}
\begin{centering}
\includegraphics[height=0.65\linewidth]{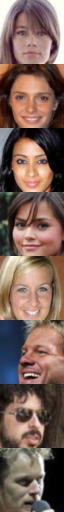}\quad{}
\includegraphics[height=0.65\linewidth]{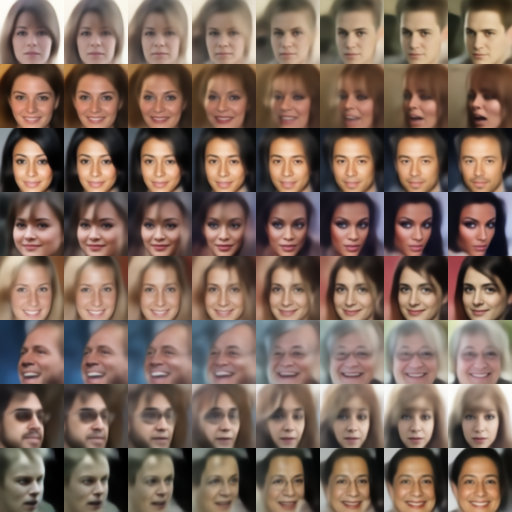}\quad{}
\includegraphics[height=0.65\linewidth]{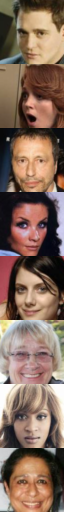}
\caption{\label{fig:celebA-interpolations}Interpolation experiments for celebA}

\par\end{centering}
\end{figure}

We also used AVB (without AC) to train a deep convolutional network on the celebA-dataset 
\cite{liu2015faceattributes} for a $64$-dimensional latent space with $\mathcal N(0,1)$-prior. For the decoder and adversary we use two 
deep convolutional neural networks acting on $x$ like in \citet{radford2015unsupervised}. We add the noise $\epsilon$ and the 
latent code $z$ to 
each hidden layer via a learned projection matrix. Moreover, in the encoder and 
decoder we use three RESNET-blocks \cite{he2015deep} at each scale of the neural network. 
We add the log-prior $\log p(z)$ explicitly to the adversary $T(x,z)$, so that it only has to learn
the log-density of the inference model $q_\phi(z \mid x)$.

The  samples for celebA are
shown in Figure \ref{fig:celebA-samples}. We see that our model produces 
visually sharp images of faces.
To demonstrate that the model has indeed learned an abstract representation of 
the data, we show reconstruction results and
the result of linearly interpolating the $z$-vector in the latent space in 
Figure \ref{fig:celebA-interpolations}. We see that the reconstructions
are reasonably sharp and the model produces realistic images for all 
interpolated $z$-values.

\paragraph{MNIST}
\begin{table}[t]

\begin{subtable}[b]{\linewidth}
\centering

\resizebox{\linewidth}{!}{
 \begin{tabular}{ l | c c c}
 
   & ELBO & AIS & reconstr. error \\
  \hline
  AVB + AC & $\approx -85.1 \pm 0.2$ & $-83.7 \pm 0.3$ & $59.3 \pm 0.2$\\
  VAE &  $-88.9 \pm 0.2$ & $-85.0 \pm 0.3$ & $62.2 \pm 0.2$\\
  auxiliary VAE&  $-88.0 \pm 0.2$ & $-83.8 \pm 0.3$ & $62.1 \pm 0.2$\\
  VAE + IAF &  $-88.9 \pm 0.2$ & $-84.9 \pm 0.3$ & $62.3 \pm 0.2$
\end{tabular}
}
\caption{fully-connected decoder ($\mathrm{dim}(z)=32$)}
\label{tab:mnist-full32}
\end{subtable}
\vspace{0.5cm}

\begin{subtable}[b]{\linewidth}
\centering
\resizebox{\linewidth}{!}{
 \begin{tabular}{ l | c c c}
 
   & ELBO & AIS & reconstr. error \\
  \hline
  AVB + AC & $\approx -93.8 \pm 0.2$ & $-89.7 \pm 0.3$ & $76.4 \pm 0.2$\\
  VAE &  $-94.9 \pm 0.2$ & $-89.9 \pm 0.4$ & $76.7 \pm 0.2$\\
  auxiliary VAE&  $-95.0\pm 0.2$ & $-89.7 \pm 0.3$ & $76.8 \pm 0.2$\\
  VAE + IAF &  $-94.4 \pm 0.2$ & $-89.7 \pm 0.3$ & $76.1 \pm 0.2$\\
\end{tabular}
}
\caption{convolutional decoder ($\mathrm{dim}(z)=8$)}
\label{tab:mnist-hvi8}
\end{subtable}
\vspace{0.5cm}

\begin{subtable}[b]{\linewidth}
\centering
\resizebox{\linewidth}{!}{
 \begin{tabular}{ l | c c c}
 
   & ELBO & AIS & reconstr. error \\
  \hline
  AVB + AC & $\approx -82.7 \pm 0.2$ & $-81.7 \pm 0.3$ & $57.0 \pm 0.2$\\
  VAE &  $-85.7 \pm 0.2$ & $-81.9 \pm 0.3$ & $59.4 \pm 0.2$\\
  auxiliary VAE&  $ -85.6 \pm 0.2$ & $-81.6 \pm 0.3$ & $ 59.6 \pm 0.2$\\
  VAE + IAF &  $-85.5 \pm 0.2$ & $-82.1 \pm 0.4$ & $59.6 \pm 0.2$\\
\end{tabular}
}
\caption{convolutional decoder ($\mathrm{dim}(z)=32$)}
\label{tab:mnist-hvi32}
\end{subtable}

\end{table}

To evaluate how AVB with adaptive contrast compares against other methods on a fixed decoder architecture, we reimplemented the methods from 
\citet{maaloe2016auxiliary} and \citet{kingma2016improving}. The method from \citet{maaloe2016auxiliary} tries to make the variational approximation to the posterior more
flexible by using auxiliary variables, the method from \citet{kingma2016improving} tries to improve the variational approximation by employing an Inverse Autoregressive Flow (IAF), a particularly flexible instance
of a normalizing flow \cite{rezende2015variational}. In our experiments, we compare AVB with adaptive contrast to a standard VAE with diagonal Gaussian inference model as well as the methods from \citet{maaloe2016auxiliary} and \citet{kingma2016improving}.

In our first experiment, we evaluate all methods on training a decoder that is given by a fully-connected neural network with ELU-nonlinearities and two hidden layers with $300$ units each. The prior distribution $p(z)$ is given by a $32$-dimensional standard-Gaussian distribution.

The results are shown in Table~\ref{tab:mnist-full32}. We observe, that both AVB and the VAE with auxiliary variables achieve a better (approximate) ELBO than a standard VAE. When evaluated using AIS, both methods result in similar log-likelihoods. However, AVB results in a better reconstruction error than an auxiliary variable VAE and a better (approximate) ELBO. We observe that our implementation of a VAE with IAF did not improve on a VAE with diagonal Gaussian inference model. We suspect that this due to optimization difficulties.

In our second experiment, we train a decoder that is given by the shallow convolutional neural network described in \citet{salimans2015markov} with $800$ units in the last fully-connected hidden layer. The prior distribution $p(z)$ is given by either a $8$-dimensional or a $32$-dimensional standard-Gaussian distribution. 
 
The results are shown in Table~\ref{tab:mnist-hvi8} and Table~\ref{tab:mnist-hvi32}. Even though AVB achieves a better (approximate) ELBO and a better reconstruction error for a $32$-dimensional latent space, all methods achieve similar log-likelihoods for this decoder-architecture, raising the question if strong inference models are always necessary to obtain a good generative model. Moreover, we found that neither auxiliary variables nor IAF did improve the ELBO. Again, we believe this is due to optimization challenges.

\end{document}